\documentclass[12pt]{article}
\usepackage[margin=1in]{geometry}
\usepackage{amsmath,amsfonts,amssymb,amsthm}
\usepackage{graphicx,float}
\usepackage{natbib}
\usepackage{algorithm,algorithmic}
\usepackage{bm}
\usepackage{booktabs}
\usepackage{hyperref}
\usepackage{setspace}
\usepackage{tikz}
\usetikzlibrary{positioning,arrows.meta,calc}

\newif\ifanonymous
\anonymousfalse  

\newtheorem{theorem}{Theorem}
\newtheorem{proposition}[theorem]{Proposition}
\newtheorem{definition}[theorem]{Definition}
\newtheorem{remark}[theorem]{Remark}

\newcommand{\R}{\mathbb{R}}
\newcommand{\E}{\mathbb{E}}
\newcommand{\calX}{\mathcal{X}}

\graphicspath{{fig/}}

\title{Generative Bayesian Computation as a Scalable Alternative\\
       to Gaussian Process Surrogates}

\ifanonymous
  \author{Anonymous}
  \date{}
\else
  \author{
    Nick Polson\footnote{Booth School of Business, University of Chicago. \texttt{ngp@chicagobooth.edu}.} \and Vadim Sokolov\footnote{Department of Systems Engineering and Operations Research, George Mason University. \texttt{vsokolov@gmu.edu}.} }
  \date{February 2026}
\fi

\begin{document}

\maketitle

\begin{abstract}
\noindent
Gaussian process (GP) surrogates are the default tool for emulating expensive computer experiments, but cubic cost, stationarity assumptions, and Gaussian predictive distributions limit their reach. We propose Generative Bayesian Computation (GBC) via Implicit Quantile Networks (IQNs) as a surrogate framework that targets all three limitations. GBC learns the full conditional quantile function from input--output pairs; at test time, a single forward pass per quantile level produces draws from the predictive distribution.

Across fourteen benchmarks we compare GBC to four GP-based methods. GBC improves CRPS by 11--26\% on piecewise jump-process benchmarks, by 14\% on a ten-dimensional Friedman function, and scales linearly to 90,000 training points where dense-covariance GPs are infeasible. A boundary-augmented variant matches or outperforms Modular Jump GPs on two-dimensional jump datasets (up to 46\% CRPS improvement). In active learning, a randomized-prior IQN ensemble achieves nearly three times lower RMSE than deep GP active learning on Rocket LGBB. Overall, GBC records a favorable point estimate in 12 of 14 comparisons. GPs retain an edge on smooth surfaces where their smoothness prior provides effective regularization.
\end{abstract}


\newpage

\section{Introduction}\label{sec:intro}

A surrogate model --- also called an emulator or meta-model --- approximates an expensive computer simulation $y = f(x)$, $x \in \R^d$, from a modest training campaign of input--output pairs \citep{sacks1989design,santner2018design,gramacy2020surrogates}. Such emulators are central to calibration \citep{kennedy2001bayesian,schultz2022bayesian}, sensitivity analysis \citep{oakley2004probabilistic}, Bayesian optimization \citep{jones1998efficient,schultz2018bayesian} --- for which deep reinforcement learning offers a model-free alternative \citep{schultz2018deep} --- sequential design \citep{morris1995exploratory,gramacy2009adaptive}, and active learning \citep{mackay1992information,cohn1994neural}. Large-scale agent-based simulators \citep{sokolov2012flexible} exemplify the setting: a single evaluation can take minutes to hours, making GP-based Bayesian optimization or calibration the dominant approach. Gaussian processes (GPs) have become the default surrogate choice because they furnish analytic predictive means and variances and support coherent Bayesian uncertainty quantification \citep{rasmussen2006gaussian,gramacy2020surrogates}.

The GP framework continues to be extended in important directions. \citet{sauer2023active} combine deep GPs with elliptical slice sampling and active learning via the ALC criterion to handle non-stationary response surfaces in small-to-medium simulation campaigns, building on the deep GP formulation of \citet{damianou2013deep}; \citet{schultz2022deep} extend deep GPs to heteroskedastic, high-dimensional outputs. \citet{holthuijzen2025synthesizing} apply bias-corrected GP surrogates with Vecchia approximations \citep{katzfuss2021general} to large ecological datasets ($n \approx 8.4 \times 10^6$), demonstrating that elaborate multi-stage pipelines are needed to scale GPs beyond $n \approx 10{,}000$. For response surfaces with sudden discontinuities, \citet{flowers2026modular} propose Modular Jump GPs (MJGP) that EM-cluster the marginal responses to identify regimes, train a classifier to estimate regime probabilities, and fit a single GP on augmented inputs. \citet{park2026active} develop active learning criteria for piecewise Jump GP surrogates, demonstrating that bias near regime boundaries must enter the acquisition criterion. Local GP methods --- including the LAGP approach of \citet{gramacyapley2015local} and its large-scale extension by \citet{cole2022large} --- address the cubic cost barrier by fitting separate GP neighborhoods, while \citet{cooper2026modernizing} extend fully Bayesian GP classification to large $n$ via Vecchia approximation combined with ESS.

Each advance addresses a genuine limitation, but each solution is problem-specific: deep GPs require MCMC; Vecchia approximations impose a conditional independence graph; jump modifications need cluster assignment followed by local GP inference; and active learning criteria must be re-derived for each new surrogate class. A practitioner facing a new computer experiment must diagnose which limitation applies and select the corresponding specialist extension --- or else accept the limitations of a standard stationary GP.

Quantile neural network surrogates offer a complementary approach \citep{polson2020deep}; deep learning methods for dimensionality reduction \citep{polson2021deep} further enable scaling to high-dimensional inputs. Quantile regression \citep{koenker2005quantile,gasthaus2019probabilistic} provides a direct route to conditional distributions, and recent work in simulation-based inference (SBI) has shown that neural networks trained on simulated data can perform amortized posterior inference \citep{cranmer2020frontier,papamakarios2021normalizing,lueckmann2021benchmarking}. The Implicit Quantile Network \citep{dabney2018implicit} embeds the quantile level as a continuous input, enabling a single network to represent the entire conditional quantile function --- a property we exploit for surrogate modeling.

GBC was introduced by \citet{polson2025generative} for amortized Bayesian inference. Here we develop GBC with IQNs as a surrogate framework that targets all three GP limitations simultaneously. Training scales as $\mathcal{O}(N)$ via stochastic gradient descent; the IQN delivers the full conditional quantile function (not just mean and variance); and the same architecture handles non-stationarity, heteroskedasticity, and jump discontinuities without custom kernel or likelihood specifications --- adaptation requires only loss-weight selection and, for jump boundaries, an optional input augmentation (Section~\ref{sec:gbc-aug}). Once trained, predictive samples cost $\mathcal{O}(1)$ per new input versus $\mathcal{O}(n)$ per GP predictive mean or $\mathcal{O}(n^2)$ per predictive variance, in addition to the $\mathcal{O}(n^3)$ GP training cost. We compare GBC to four GP-based methods --- \texttt{hetGP} \citep{binois2018practical}, stationary GP, MJGP \citep{flowers2026modular}, and DGP+ALC \citep{sauer2023active} --- across fourteen benchmarks spanning $d = 1$ to $d = 10$ and $n = 133$ to $n = 90{,}000$. The benchmarks are designed so that each GP method is tested under conditions favorable to it: \texttt{hetGP} on heteroskedastic data, MJGP on jump boundaries, DGP+ALC on active learning. The pattern that emerges is that GBC's advantage grows with problem dimension, training-set size, and boundary sharpness, while GPs retain a clear edge on smooth, moderate-$n$ surfaces where their kernel prior provides effective regularization. Table~\ref{tab:summary} consolidates the fourteen head-to-head comparisons.

GBC does have clear limitations. It requires a generative forward model from which joint parameter--data samples $(\theta, y)$ can be drawn; when only a small fixed dataset is available, augmentation is needed. On smooth, stationary problems at small sample sizes, a well-tuned GP with analytic posteriors can match or exceed GBC accuracy. On heteroskedastic data, purpose-built models such as \texttt{hetGP} achieve better predictive calibration (coverage) when the Gaussian heteroskedastic assumption holds. The advantage of GBC grows with dimensionality, scale ($n$), and the sharpness of any regime boundary --- precisely the settings where standard GP surrogates struggle most.

Our contributions are as follows. First, we formulate GBC with IQNs as a unified surrogate framework (Definition~\ref{def:gbc}) that addresses the three main limitations of GP surrogates --- cubic scaling, stationarity, and Gaussian predictive distributions --- within a single architecture. The theoretical foundation rests on the noise outsourcing theorem (Theorem~\ref{thm:noise-outsourcing}), which guarantees that any conditional distribution can be represented as a deterministic function of the conditioning variable and a uniform random variable; we prove that the IQN architecture is a universal approximator of this map (Proposition~\ref{prop:uat}) and establish CRPS convergence rates (Propositions~\ref{prop:crps}--\ref{prop:consistency}). Second, we introduce a three-term training loss that combines an $L_1$ location anchor, a quantile-ordering surrogate, and the quantile loss, with weights that can be adapted to the problem structure: a single configuration change (no architectural modification) yields a 28\% CRPS improvement on jump boundaries. Third, we develop GBC-Aug, a boundary-augmented variant that couples EM-based regime detection with the IQN backend, matching or surpassing Modular Jump GPs on their own benchmarks while training 26$\times$ faster. Fourth, we propose a randomized-prior IQN ensemble for active learning that provides reliable acquisition signals where standard NN uncertainty methods (bootstrap, SNGP, anchored ensembles, last-layer Laplace) fail. Fifth, we provide a systematic empirical comparison across fourteen benchmarks and four GP-based competitors, with 10--50 Monte Carlo replicates per benchmark, establishing when GBC outperforms GPs and when it does not.

Section~\ref{sec:background} reviews GP surrogates and their structural limitations. Section~\ref{sec:gbc} introduces GBC: the noise outsourcing theorem, IQN architecture, and training loss. Section~\ref{sec:extensions} develops boundary augmentation for jump processes and active learning within the GBC framework. Section~\ref{sec:benchmarks} describes the implementation and presents synthetic benchmarks. Section~\ref{sec:real} evaluates GBC on real-data and applied examples. Section~\ref{sec:discussion} discusses the results and open problems.

\section{Gaussian Process Surrogates}\label{sec:background}

This section reviews GP surrogates and their limitations, setting up the three challenges --- cubic scaling, stationarity, and Gaussian predictive distributions --- that GBC addresses.

\subsection{GP regression}

Consider a black-box simulator $y = f(x) + \epsilon$, $x \in \calX \subseteq \R^d$, $\epsilon \sim \mathcal{N}(0, \sigma^2)$. A GP surrogate places a GP prior on $f$: for any finite collection $X_n = (x_1, \ldots, x_n)^\top$,
\[
  Y_n \mid X_n \sim \mathcal{N}_n(\mu, \Sigma(X_n)), \quad \Sigma_{ij}(X_n) = k(x_i, x_j) + g\,\mathbb{I}_{\{i=j\}},
\]
where $k(\cdot,\cdot)$ is a covariance (kernel) function, typically a function of the pairwise distance $\|x_i - x_j\|$, and $g > 0$ is a nugget parameter \citep{rasmussen2006gaussian}. A common choice is the squared-exponential kernel $k(x_i,x_j) = \sigma_f^2 \exp(-\|x_i - x_j\|^2/\ell)$; the GP baselines in Sections~\ref{sec:benchmarks}--\ref{sec:real} use a Mat\'{e}rn-$5/2$ kernel with separable (per-dimension) length scales, which is less smooth than the SE kernel and standard in computer experiments \citep{gramacy2020surrogates}.  Posterior prediction at test locations $X_*$ has closed form (kriging equations), with predictive mean $\mu_Y(X_*)$ and covariance $\Sigma_Y(X_*)$; see \citet{gramacy2020surrogates} for a detailed exposition. Inference for $(\sigma_f^2, \ell, g)$ requires $\mathcal{O}(n^3)$ flops per likelihood evaluation.

\subsection{Limitations and extensions}

Three structural limitations constrain the applicability of GP surrogates.  The most immediate is computational: dense covariance matrices scale as $\mathcal{O}(n^2)$ in storage and $\mathcal{O}(n^3)$ in Cholesky decomposition, limiting plain GPs to $n \ll 10{,}000$ \citep{cooper2026modernizing}. \citet{holthuijzen2025synthesizing} require Vecchia approximations \citep{katzfuss2021general} to handle $n \approx 8.4 \times 10^6$, while local GP methods \citep{gramacyapley2015local,cole2022large} mitigate the cost by fitting separate neighborhoods. Sparse variational GPs \citep{titsias2009variational,hensman2013gaussian} reduce cost to $\mathcal{O}(nm^2)$ via $m \ll n$ inducing points, at the cost of an approximate posterior. All of these approaches introduce structural assumptions (conditional independence graphs, neighborhood sizes, or inducing-point placement) that add tuning burden.

The second limitation is stationarity. The kernel $\Sigma(x_i, x_j)$ depends only on $\|x_i - x_j\|$, encoding the assumption that correlation structure is spatially uniform. Non-stationary dynamics --- regime changes, heteroskedasticity, discontinuities --- require bespoke solutions: DGP warping \citep{sauer2023active,damianou2013deep}, tree-based partitions \citep{gramacy2008bayesian}, or explicit jump modeling \citep{flowers2026modular, park2026active}. Each modification introduces additional inference complexity.

The third limitation is distributional. Standard GP prediction delivers a Gaussian $\mathcal{N}(\mu_Y, \Sigma_Y)$. Downstream tasks requiring quantile estimates, tail probabilities, or samples from a non-Gaussian posterior receive only a Gaussian approximation. Extending GPs to classification requires a latent layer and MCMC or variational integration \citep{cooper2026modernizing}, specifically because the Bernoulli likelihood is not conjugate to the GP prior.

Figure~\ref{fig:jump} provides a simple illustration.  On a 1D function with a jump discontinuity, the stationary GP must choose a single length scale that trades off smoothness within each regime against accuracy near the boundary.  The IQN-based GBC, by contrast, adapts its conditional quantile bands locally --- widening them at the jump and tightening them away from it.  The full piecewise benchmarks in Sections~\ref{sec:bgp}--\ref{sec:flowers} demonstrate this advantage in noisy, higher-dimensional settings.

\begin{figure}[H]
  \centering
  \includegraphics[width=\textwidth]{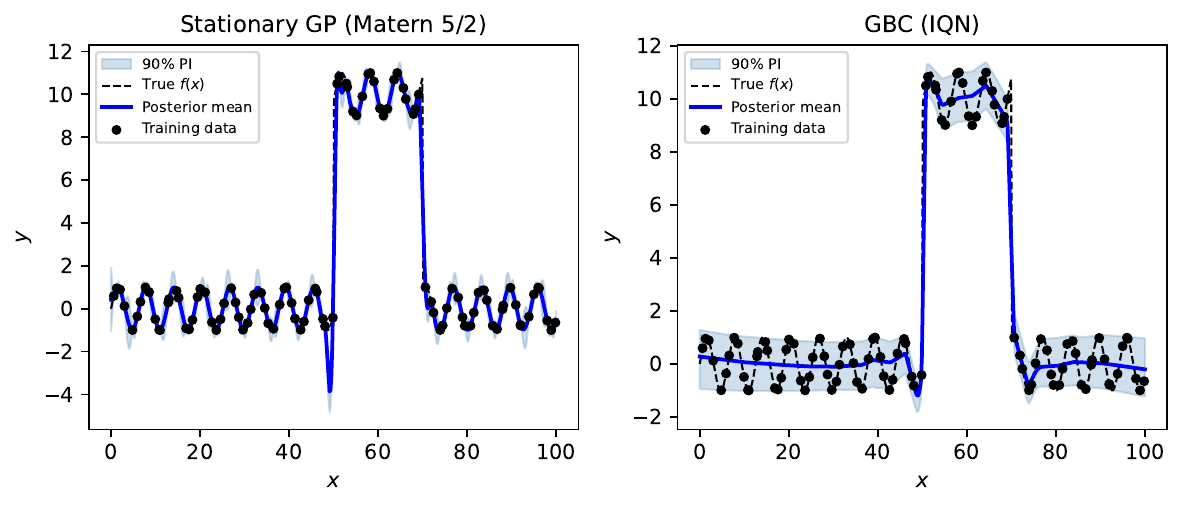}
  \caption{GP vs.\ GBC on a 1D jump function ($n = 100$, noiseless).
    The stationary GP (left) adapts its length scale near the discontinuity but
    blurs the jump; GBC (right) widens its quantile bands at the boundary.
    Details on the IQN architecture are in Section~\ref{sec:gbc}.}
  \label{fig:jump}
\end{figure}

These extensions solve individual limitations but fragment the toolkit.  A practitioner using Vecchia for scalability cannot easily combine it with jump detection for discontinuities; a deep GP designed for non-stationarity does not automatically produce calibrated quantile predictions.  GBC targets all three limitations within a single architecture, as the next section describes.

\section{Generative Bayesian Computation}\label{sec:gbc}

GBC replaces the GP with a neural network that maps inputs and a uniform random variable to predictive samples. We first present the theoretical foundation (noise outsourcing), then the IQN architecture and training loss, and conclude with approximation guarantees.

\subsection{Noise outsourcing}

The key insight behind GBC is that \emph{any} posterior distribution can be represented as a deterministic function of the data and a single uniform random variable --- no likelihood evaluation required. This follows from a classical result in probability theory.

Let $\theta \in \Theta \subseteq \R^k$ denote unknown parameters (or any quantities of interest), $y \in \mathcal{Y} \subseteq \R^p$ observed data, and $\tau \in [0,1]^k$ a uniform base variable. The goal is to draw samples from the posterior $\pi(\theta \mid y)$, whether or not a likelihood $p(y \mid \theta)$ is available.

\begin{theorem}[Noise Outsourcing; \citealp{kallenberg1997foundations}, Thm.~5.10]
\label{thm:noise-outsourcing}
If $(X, Z)$ are random variables in a Borel space $\mathcal{X} \times \mathcal{Z}$, there exist a random variable $U \perp\!\!\!\perp X$ and a measurable function $G^\star : [0,1] \times \mathcal{X} \to \mathcal{Z}$ such that
\[
  (X, Z) \stackrel{a.s.}{=} \bigl(X,\, G^\star(U, X)\bigr).
\]
\end{theorem}

In words: the dependence of $Z$ on $X$ can be captured entirely by a deterministic function that takes $X$ and an independent uniform draw as inputs. The theorem applies to any pair of jointly distributed random variables. In the Bayesian inference setting, take $X = y$ (data) and $Z = \theta$ (parameters): the generator $G^\star(\tau, y)$ is a measurable transport map from the uniform base to the posterior $\pi(\theta \mid y)$.  When $\theta$ is scalar ($k{=}1$), this reduces to the inverse CDF $F^{-1}_{\theta \mid y}(\tau)$; for $k > 1$, $G^\star$ is a general measurable function, not a component-wise quantile inversion. In the surrogate emulation setting --- the primary focus of this paper --- take $X = x$ (inputs) and $Z = y$ (scalar response): the generator $G^\star(\tau, x) = Q_\tau(Y \mid x)$ is the conditional quantile function of $Y$ given $x$.  In both cases, the IQN learns an approximation $G_\phi \approx G^\star$ from training pairs. Throughout Sections~\ref{sec:benchmarks}--\ref{sec:real}, $x$ denotes the simulator input and $y$ the response; the $(x, y)$ surrogate notation is used consistently. Further details on the proof and convergence theory appear in Supplement~A. When the conditioning variable is high-dimensional ($y \in \R^p$ with large $p$), a sufficient summary statistic $S(y)$ can replace $y$ as input; Supplement~B discusses this dimension reduction.

\subsection{Implicit Quantile Network}

We follow \citet{dabney2018implicit} in embedding the quantile level $\tau$ via a cosine transform before merging with the covariate representation:
\[
  \phi(\tau) = \bigl[\cos(j\pi\tau)\bigr]_{j=0}^{n_h-1} \in \R^{n_h},
\]
where $n_h = 32$ is the embedding dimension. The IQN computes
\[
  G_\phi(\tau, x) = f_{\rm out}\!\left(f_1\!\left(f_x(x) \odot f_\tau(\phi(\tau))\right)\right),
\]
where $f_x, f_\tau, f_1$ are fully-connected layers with ReLU activations, $\odot$ denotes elementwise product, and $f_{\rm out}$ outputs two values: a location estimate $\hat\mu$ (used only during training as an $L_1$ anchor targeting the conditional median) and a quantile estimate $\hat q_\tau$ (the predictive output at test time).

The architecture works as follows. The input $x$ passes through a feature network $f_x$ that produces a representation vector; independently, the quantile level $\tau$ is embedded via $\phi(\tau)$ and projected through $f_\tau$. The elementwise product $f_x(x) \odot f_\tau(\phi(\tau))$ allows the quantile level to modulate which features of $x$ are emphasized, enabling the network to represent different quantile surfaces as smooth functions of $\tau$. Because $\tau$ enters as a continuous input rather than a discrete index, the IQN can evaluate any quantile in $[0,1]$ at test time without retraining.

Figure~\ref{fig:iqn-arch} provides a diagram of the IQN architecture.  The left branch processes the input $x$ through a feature network $f_x$; the right branch embeds $\tau$ via the cosine transform and projects through $f_\tau$. Their elementwise product merges covariate information with quantile-level modulation, and a final layer produces both a location estimate $\hat\mu$ and the quantile prediction $\hat q_\tau$.

\begin{figure}[ht!]
\centering
\begin{tikzpicture}[
    node distance=0.9cm and 1.6cm,
    box/.style={draw, rounded corners=2pt, minimum width=1.8cm,
                minimum height=0.6cm, font=\small, align=center},
    io/.style={draw, rounded corners=2pt, minimum width=1.4cm,
               minimum height=0.6cm, font=\small, fill=gray!12},
    arr/.style={-{Stealth[length=5pt]}, thick},
  ]
  \node[io] (x) {$x \in \R^d$};
  \node[io, right=3.2cm of x] (tau) {$\tau \sim U[0,1]$};

  \node[box, below=of x] (fx) {$f_x$: FC+ReLU};
  \node[box, below=of tau] (cos) {$\phi(\tau)$: cosine};
  \node[box, below=of cos] (ftau) {$f_\tau$: FC+ReLU};

  \node[box, below=2.0cm of $(fx)!0.5!(ftau)$, minimum width=2.4cm,
        fill=blue!8] (merge) {$\odot$ (elementwise)};

  \node[box, below=of merge, minimum width=2.4cm] (f1) {$f_1$: FC+ReLU};

  \node[box, below left=0.9cm and 0.3cm of f1, fill=green!10] (mu) {$\hat\mu$};
  \node[box, below right=0.9cm and 0.3cm of f1, fill=orange!12] (q) {$\hat q_\tau$};

  \draw[arr] (x) -- (fx);
  \draw[arr] (tau) -- (cos);
  \draw[arr] (cos) -- (ftau);
  \draw[arr] (fx) -- (merge);
  \draw[arr] (ftau) -- (merge);
  \draw[arr] (merge) -- (f1);
  \draw[arr] (f1) -- (mu);
  \draw[arr] (f1) -- (q);
\end{tikzpicture}
\caption{IQN architecture.  The input $x$ and quantile level $\tau$ are
  processed through separate branches ($f_x$ and $f_\tau \circ \phi$),
  merged via elementwise product, and decoded into a location estimate
  $\hat\mu$ and a quantile prediction $\hat q_\tau$.  The cosine embedding
  $\phi(\tau)$ enables smooth interpolation over quantile levels.}
\label{fig:iqn-arch}
\end{figure}
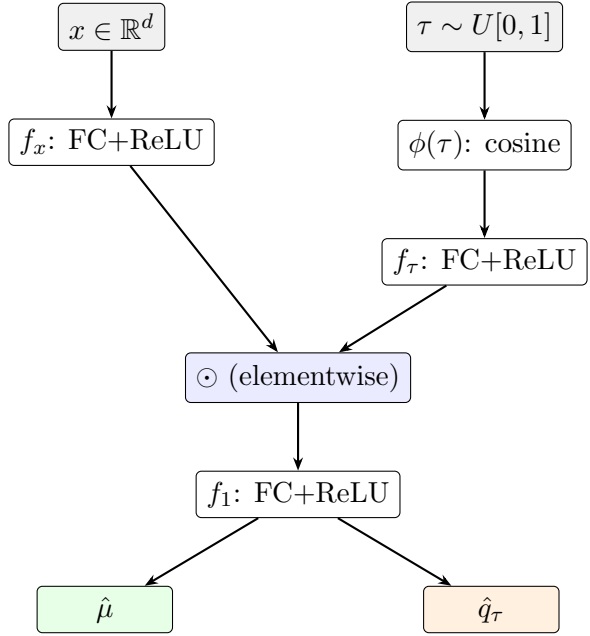

\subsection{Training loss}
\label{sec:loss}

The IQN is trained with a three-term loss that addresses three potential failure modes: mode collapse, quantile mis-ordering, and miscalibration. Let $\mathcal{D}_N = \{(x^{(i)}, y^{(i)})\}_{i=1}^N$ be training data from the forward model (in the benchmarks of Sections~\ref{sec:benchmarks}--\ref{sec:real}, $N = n$, the same training set used by the GP). For each mini-batch and draw $\tau \sim \mathrm{Uniform}[0,1]$, the loss is
\begin{equation}\label{eq:loss}
  \ell(\tau) = w_1 \underbrace{\E|y - \hat\mu|}_{\text{L1 location anchor}}
  + w_2 |{\tau - 0.5}| \underbrace{\E\bigl[m_\tau\bigr]}_{\text{ordering surrogate}}
  + w_3 \underbrace{\E\bigl[\max\bigl(\tau\,e,\,(\tau{-}1)e\bigr)\bigr]}_{\text{quantile}},
\end{equation}
where $e = y - \hat q_\tau$ and the ordering surrogate is
\[
  m_\tau = \begin{cases} \max(0,\; \hat q_\tau - y) & \text{if } \tau < 0.5, \\ \max(0,\; y - \hat q_\tau) & \text{if } \tau \geq 0.5.
  \end{cases}
\]
Here $y$ is the response (prediction target) and $x$ the input (conditioning variable); for the Bayesian inference instantiation (Section~\ref{sec:gbc}, noise outsourcing), replace $y \to \theta$ and $x \to y$.

The first term (location anchor) prevents mode-collapse by encouraging the network's location output to track the conditional median. The second term is a \emph{quantile-ordering surrogate}: for $\tau < 0.5$, it penalizes when $\hat q_\tau$ exceeds $y$; for $\tau \geq 0.5$, it penalizes when $\hat q_\tau$ falls below $y$.  This is not a direct non-crossing constraint between quantile levels $\tau_1 < \tau_2$ (as in \citealp{cannon2018non}), but rather a soft non-crossing penalty: if lower quantiles stay below $y$ and upper quantiles stay above $y$, they are less likely to cross.  This soft penalty works well in practice. The $|\tau - 0.5|$ weighting strengthens the penalty at extreme quantile levels.  Quantile crossings are not formally prevented but were not observed to degrade CRPS in any benchmark; a strict non-crossing guarantee would require isotonic post-processing or the constrained architecture of \citet{cannon2018non}. The third term is the standard quantile loss \citep{koenker2005quantile}, which is minimized at the true conditional quantile.

The default weights $(w_1, w_2, w_3) = (0.3, 0.3, 0.4)$ work well for smooth and heteroskedastic responses. For jump-process benchmarks with sharp discontinuities, reducing the $L_1$ anchor to $(0.1, 0.2, 0.7)$ (``quantile-dominant'') improves CRPS by freeing the network from smoothing across boundaries; these two configurations are fixed a priori (not tuned on test data) and reported in all tables. In practice, a simple diagnostic guides the choice: if the marginal response distribution is bimodal or the domain contains known regime boundaries, quantile-dominant weights are preferred; otherwise the default weights are used.  This diagnostic inspects the training responses and is therefore data-dependent, though it uses only the marginal distribution (not test data).  Table~\ref{tab:flowers} below shows that default weights on the jump benchmarks (Phantom, Star) yield CRPS $3$--$4\times$ worse than quantile-dominant, confirming that the weight choice materially affects performance on these problems. We train with Adam \citep{kingma2015adam} and cosine annealing \citep{loshchilov2017sgdr}.

\subsection{GBC algorithm}

\begin{algorithm}
\caption{Generative Bayesian Computation (GBC) for Surrogate Modeling}\label{alg:gbc}
\begin{algorithmic}[1]
\STATE \textbf{Train:}
  \STATE\quad Obtain training pairs $(x^{(i)}, y^{(i)})$ from the forward model,
         $i = 1, \ldots, N$.
  \STATE\quad Train IQN to obtain $G_{\hat\phi}$ (the fitted network) by minimizing~\eqref{eq:loss} on
         $\{(x^{(i)}, y^{(i)}, \tau^{(i)})\}$ via Adam + cosine annealing.
\STATE \textbf{Test:} Given new input $x^*$:
  \STATE\quad Draw $\tau^{(b)} \sim \mathrm{Uniform}[0,1]$ for
         $b = 1, \ldots, B$.
  \STATE\quad Return $\hat y^{(b)} = \hat q_{\tau^{(b)}}$ from $G_{\hat\phi}(\tau^{(b)},\, x^*)$
         as predictive samples (the $\hat\mu$ head is not used at test time).
\end{algorithmic}
\end{algorithm}

Training GBC costs $\mathcal{O}(N \cdot C_{\rm fwd})$ where $C_{\rm fwd}$ is the cost of one network forward pass. For a GP with $n$ training points, hyperparameter estimation requires $\mathcal{O}(n^3)$ flops per likelihood evaluation (Cholesky factorization of the $n \times n$ covariance matrix). Predictive mean at a new point costs $\mathcal{O}(n^2)$ from scratch, or $\mathcal{O}(n)$ if the Cholesky factor is cached; the full predictive variance requires an additional $\mathcal{O}(n^2)$ solve. On the Friedman $d{=}10$ benchmark at $n = 2{,}000$, GP training requires $47.6$\,s (CPU) versus $17.7$\,s for GBC (GPU), with GBC also achieving lower RMSE (Table~\ref{tab:friedman}); absolute timings reflect hardware differences (Section~\ref{sec:friedman} separates asymptotic scaling). Beyond $n = 2{,}000$ (the largest size tested for the GP baseline on this hardware), the dense-covariance GP becomes infeasible while GBC scales linearly. Once trained, generating $B$ predictive samples at any new input $x^*$ costs $\mathcal{O}(B \cdot C_{\rm fwd})$ --- no matrix factorization, no MCMC chain.

\subsection{Unified GBC framework}
\label{sec:gbc-framework}

The following definition unifies the variants used throughout this paper.

\begin{definition}[GBC Framework]\label{def:gbc}
Given training pairs $\{(x_i, y_i)\}_{i=1}^N$ from a forward model, GBC proceeds in three steps: (i)~compute features $z = T(x)$ via a (possibly identity) transformation $T\colon \R^d \to \R^p$; (ii)~train an IQN $G_\phi(\tau, z)$ to approximate $Q_\tau(Y \mid Z{=}z)$ via the three-term loss~\eqref{eq:loss}; and (iii)~at a new input $x^*$ with $z^* = T(x^*)$, sample $\tau_1, \ldots, \tau_B \stackrel{\text{iid}}{\sim} U(0,1)$ and return $\{G_\phi(\tau_b, z^*)\}_{b=1}^B$ as predictive samples.
\end{definition}

\noindent
The two variants in Section~\ref{sec:benchmarks} are special cases: \textbf{GBC} (default) sets $T(x) = x$; \textbf{GBC-Aug} sets $T(x) = [x,\, \hat{c}_\psi(x)]$ where $\hat{c}_\psi\colon \calX \to [0,1]$ is a learned boundary classifier (Section~\ref{sec:gbc-aug}).

The next three propositions establish approximation, accuracy, and consistency guarantees for IQN quantile estimation under the quantile (check) loss.  In practice, GBC minimizes the three-term loss~\eqref{eq:loss} rather than the quantile loss alone; Remark~\ref{rem:composite} below discusses the relationship between the theoretical and practical objectives.

\begin{proposition}[Universal approximation of conditional quantiles]
\label{prop:uat}
Assume $\calX \subset \R^d$ is compact and the conditional quantile function $(\tau, x) \mapsto Q_\tau(Y \mid x)$ is continuous on $[0,1] \times \calX$. Then for any $\varepsilon > 0$, there exists a ReLU network $G_\phi\colon [0,1] \times \calX \to \R$ with finitely many hidden units such that
\[
  \sup_{\tau \in [0,1],\, x \in \calX} \bigl|G_\phi(\tau, x) - Q_\tau(Y \mid x)\bigr| < \varepsilon.
\]
\end{proposition}

\noindent
\textit{Proof idea.} The conditional quantile function is continuous on the compact set $[0,1] \times \calX$; the standard UAT \citep{hornik1989multilayer} gives a ReLU approximator on this joint input. The IQN's elementwise merge can realize the required joint-input network via bias-term padding. Full proof in Supplement~A.

\begin{proposition}[CRPS approximation bound]
\label{prop:crps}
Let $Q_\tau(Y \mid x)$ denote the true conditional quantile function and $G_\phi(\tau, x)$ an approximation satisfying $\sup_{\tau \in [0,1]} |G_\phi(\tau, x) - Q_\tau(Y \mid x)| \leq \varepsilon$ for a given $x$.  Then for any observation $y$,
\[
  \bigl|\mathrm{CRPS}(G_\phi, y) - \mathrm{CRPS}(Q, y)\bigr| \leq 2\varepsilon.
\]
\end{proposition}

\noindent
\textit{Proof idea.} The CRPS equals $2\int_0^1 \rho_\tau(y - F^{-1}(\tau))\,d\tau$ \citep{gneiting2007strictly}, and the check loss $\rho_\tau$ is Lipschitz-1 in its quantile argument. Integrating the pointwise bound gives $2\varepsilon$. Full proof in Supplement~A.

\begin{proposition}[Consistency]
\label{prop:consistency}
Assume $\calX \subset \R^d$ is compact, the response $Y$ is bounded, the conditional quantile function $Q_\tau(Y \mid x)$ is $\beta$-H\"{o}lder continuous in $(x, \tau)$ for some $\beta > 0$, and the network class $\mathcal{G}_n$ consists of ReLU networks whose depth $L_n$ and width $W_n$ satisfy $L_n W_n \log(L_n W_n) = o(n / \log n)$.  Then the empirical quantile loss minimizer $\hat{G}_n \in \mathcal{G}_n$ satisfies
\[
  \E\!\left[\int_0^1 \rho_\tau(Y - \hat{G}_n(\tau, X))\,d\tau\right] \;\longrightarrow\; \inf_G \E\!\left[\int_0^1 \rho_\tau(Y - G(\tau, X))\,d\tau\right] \quad\text{as } n \to \infty.
\]
In particular, the expected CRPS of $\hat{G}_n$ converges to the Bayes-optimal CRPS.
\end{proposition}

\noindent
\textit{Proof idea.} Propositions~\ref{prop:uat}--\ref{prop:crps} control approximation error. Estimation error is bounded by the Rademacher complexity of $\mathcal{G}_n$ composed with the Lipschitz check loss: $O(\sqrt{L_n W_n \log(L_n W_n) / n})$ \citep{padilla2022quantile}. The network-size condition ensures both terms vanish; minimax rates follow from \citet{schmidt-hieber2020nonparametric}. Full proof in Supplement~A.

\begin{remark}[Theory vs.\ practice]\label{rem:composite}
Propositions~\ref{prop:uat}--\ref{prop:consistency} analyze the quantile loss (the $w_3$ term in~\eqref{eq:loss}).  The $L_1$ anchor and ordering surrogate serve as finite-sample regularizers --- preventing mode collapse and quantile mis-ordering --- but introduce a bias controlled by $w_1$ and $w_2$. As $w_3 \to 1$, the training objective approaches the pure quantile risk and the theoretical guarantees apply directly; the quantile-dominant configuration $(0.1, 0.2, 0.7)$ used for jump benchmarks is closer to this regime. The empirical results in Sections~\ref{sec:benchmarks}--\ref{sec:real} demonstrate practical performance of the composite-loss estimator; characterizing the bias--variance tradeoff as a function of $(w_1, w_2, w_3)$ remains open. These results guarantee existence and consistency but not finite-sample calibration: IQN coverage can deviate from nominal in either direction (Section~\ref{sec:discussion}).
\end{remark}

\section{Boundary Augmentation and Active Learning}\label{sec:extensions}

GBC's IQN core handles non-stationarity by default, but two extensions improve performance in specific regimes: an input-augmentation strategy for sharp jump boundaries, and an ensemble-based acquisition criterion for active learning.

\subsection{Non-stationarity and heteroskedasticity}

Because the IQN $G_\phi$ is composed of feedforward networks that are universal approximators \citep{hornik1989multilayer}, it imposes no stationarity assumption: the map $(\tau, x) \mapsto G_\phi(\tau, x)$ can learn arbitrary response surfaces, including those with heteroskedastic variance and discontinuous structure. In contrast, a stationary GP kernel assigns positive covariance to any two points whose inputs are close, encoding a smooth, homoskedastic prior over the response surface.

For a heteroskedastic process $y = f(x) + \sigma(x)\epsilon$, $\epsilon \sim \mathcal{N}(0,1)$, a stationary GP effectively estimates an average noise level $\bar\sigma$, producing overly-wide intervals in low-variance regions and under-coverage in high-variance regions. The IQN adapts its conditional quantile estimates locally to $\sigma(x)$, producing interval widths that vary with the local noise level.  Whether the resulting coverage matches the nominal level depends on sample size and signal-to-noise ratio; Section~\ref{sec:moto} illustrates both the adaptation and the residual coverage gap on the motorcycle benchmark.

\subsection{Boundary augmentation for jump processes}
\label{sec:gbc-aug}

When the response surface contains sharp jump discontinuities, the IQN must simultaneously learn (i)~the boundary geometry and (ii)~the conditional quantile function on each side.  Learning the boundary from data alone is feasible (Section~\ref{sec:bgp}) but can be improved by providing a boundary indicator as an additional input feature.  An alternative implicit approach --- a mixture-of-experts (MoE) IQN with a learned gating network that routes inputs to $K$ expert heads --- was tested and failed: with soft routing (weighted average of experts), the gate never commits to a single expert per region (mean specialization $0.64$--$0.71$ out of $1.0$), and the mixture output is no better than a single IQN (CRPS $0.036$--$0.039$ vs.\ $0.031$); with hard routing (argmax at inference), predictions are catastrophic (CRPS $> 0.59$) because individual experts, trained only to contribute partial signals under soft mixing, cannot produce valid standalone predictions.  The softmax gate is inherently smooth and cannot represent a step function in input space, making implicit partitioning fundamentally unsuitable for sharp discontinuities.

Inspired by the modular pipeline of MJGP \citep{flowers2026modular}, we propose \emph{GBC-Aug}, a three-step augmentation. First, a two-component Gaussian mixture model (EM) is fit to the marginal training responses $\{y_i\}$, producing hard binary labels $c_i \in \{0,1\}$ that separate the two regimes. Second, an MLP classifier $\hat{f}\colon \calX \to [0,1]$ is trained via binary cross-entropy loss to predict $P(c = 1 \mid x)$ from the inputs alone. Third, the classifier output is appended as an extra input feature and a standard IQN is trained on the augmented data $\{([x_i,\, \hat{f}(x_i)],\; y_i)\}$. The classifier provides a soft boundary indicator that the IQN can exploit to localize its quantile estimates without learning the boundary geometry from scratch.  The same IQN architecture and loss function are used as in plain GBC; the only change is the additional input dimension.

GBC-Aug applies when the response has binary jump structure (two distinct regimes). It is not needed for smooth boundaries (e.g., AMHV) or when no specialist competitor exists (e.g., Michalewicz at $n = 90{,}000$). The binary assumption is limiting: if the response has more than two regimes or if the marginal modes overlap substantially, EM clustering may misassign labels, propagating errors to the classifier and ultimately to the IQN.  The Michalewicz benchmark already illustrates a partial failure of this kind (its marginal is not bimodal), which is why plain GBC is used there instead. Classifier accuracy is the key bottleneck: on Phantom, replacing the final three-layer MLP (99.98\% accuracy, CRPS~$= 0.009$) with a single-hidden-layer classifier (82\% accuracy) degrades GBC-Aug CRPS to $0.024$ (Section~\ref{sec:flowers}).

\subsection{Active learning with GBC}
\label{sec:gbc-al}

Active learning (AL) for surrogates sequentially selects the next input $x_{n+1}$ to minimize a global prediction error criterion \citep{mackay1992information, cohn1994neural, sauer2023active}. At any test location $x$, GBC generates $B$ predicted outputs $\hat{y}^{(b)} = G_\phi(\tau^{(b)}, x)$, $\tau^{(b)} \sim U(0,1)$, $b = 1, \ldots, B$. The Monte Carlo estimate of predictive variance is
\[
  \hat\sigma^2_{\rm GBC}(x) = \frac{1}{B-1}\sum_{b=1}^B \bigl(\hat{y}^{(b)} - \bar{\hat y}\bigr)^2, \quad \bar{\hat y} = B^{-1} \sum_{b=1}^B \hat y^{(b)}.
\]
More generally, GBC provides the full predictive distribution $P(\hat y \mid x)$, enabling a family of acquisition criteria.  One natural choice is the width of a predictive interval:
\begin{equation}\label{eq:al-criterion}
  x_{n+1}^* = \arg\max_{x \in \calX}\;
  \bigl[F^{-1}_{\hat y \mid x}(0.95) - F^{-1}_{\hat y \mid x}(0.05)\bigr],
\end{equation}
which generalizes the GP's variance-based ALC criterion to non-Gaussian distributions.  In practice, however, we found that \emph{ensemble disagreement} is more robust than~\eqref{eq:al-criterion} in sequential settings, because the single-model quantile spread can collapse as the IQN overfits the growing training set. Ensemble disagreement is defined as
\[
  a(x) = \sqrt{\frac{1}{K-1}\sum_{k=1}^{K}\bigl(\tilde{y}_k(x) - \bar{\tilde{y}}(x)\bigr)^2}, \quad \tilde{y}_k(x) = \mathrm{median}_b\; G_{\phi_k}(\tau^{(b)}, x),
\]
where $\tilde{y}_k(x)$ is the median prediction from ensemble member~$k$ and $\bar{\tilde{y}} = K^{-1}\sum_k \tilde{y}_k$.  All active learning experiments in Sections~\ref{sec:al-rocket}--\ref{sec:al-sat} use ensemble disagreement as the acquisition criterion.

The ensemble uses a $K{=}3$ randomized prior architecture \citep{osband2018randomized}: each IQN member $k$ computes $G_{\phi_k}(\tau,x) = g_{\phi_k}(\tau,x) + \alpha\, g_{0,k}(\tau,x)$, where $g_{0,k}$ is a frozen random-weight network (scale $\alpha = 0.5$) whose parameters are drawn once at initialization and never updated.  This design addresses a fundamental difference between NN ensembles and GP posteriors.  GP posterior samples are draws from a stochastic process whose diversity is guaranteed by the posterior covariance: even as $n \to \infty$, the samples remain distinct wherever posterior uncertainty is nonzero.  By contrast, NN ensemble members trained on the same data with different random seeds converge to nearly identical predictions as $n$ grows, because SGD drives all members toward the same empirical risk minimizer.  Bootstrap resampling partially restores diversity by varying the training set, but in our experiments bootstrap ensembles were consistently the worst acquisition strategy --- 17\% worse than random selection on the 7D Proxy and 30\% worse on SatMini (Supplement~D) --- because the resampled training sets at small $n$ produce degenerate fits.  Randomized priors inject \emph{structural} diversity that persists regardless of $n$: in data-sparse regions, $g_{\phi_k}$ cannot learn to cancel the random $g_{0,k}$, so different members make different predictions and disagreement is high; in data-rich regions, $g_{\phi_k}$ dominates and the prior perturbation becomes negligible.  Among the alternatives tested --- bootstrap ensembles, SNGP \citep{liu2020simple}, anchored ensembles \citep{pearce2020uncertainty}, and last-layer Laplace --- randomized priors produced the best acquisition signal on every benchmark (Supplement~D).

\section{Implementation and Synthetic Benchmarks}\label{sec:benchmarks}

We first describe the implementation shared by all benchmarks, then evaluate GBC on three synthetic problems of increasing complexity: a heteroskedastic 1D dataset (motorcycle), piecewise-stationary processes in $d = 2$--$4$ (BGP), and a high-dimensional scaling study ($d = 10$, Friedman).

\subsection{Implementation details}

All GBC models are implemented in PyTorch \citep{paszke2019pytorch}. Code to reproduce every experiment in this paper is available at
\ifanonymous
[URL removed for blind review].
\else
\url{https://github.com/vadimsokolov/gbc-surrogate}.
\fi
The IQN uses the same architecture throughout: the feature network $f_x$ and quantile projection $f_\tau$ each contain one hidden layer of 256 units (ReLU), followed by a shared hidden layer $f_1$ of 256 units after the elementwise merge, with quantile embedding dimension $n_h = 32$ and the three-term loss~\eqref{eq:loss} with default weights $(0.3, 0.3, 0.4)$ unless otherwise noted.  Training uses Adam \citep{kingma2015adam} with learning rate $10^{-3}$ and cosine annealing \citep{loshchilov2017sgdr}. Epochs range from $3{,}000$ to $8{,}000$ depending on dataset size and are specified per benchmark below; all were fixed before seeing test-set results. The architecture (depth, width, $n_h$) was not tuned per benchmark; the same configuration is used for all fourteen datasets, ranging from $d{=}1$, $n{=}133$ (motorcycle) to $d{=}10$, $n{=}90{,}000$ (Michalewicz).

Two variants of the GBC framework appear in this paper, differing only in the input representation. \textbf{GBC} (default) feeds the raw inputs $x$ to the IQN with the three-term loss~\eqref{eq:loss} and is used for all synthetic benchmarks and active learning experiments. \textbf{GBC-Aug} appends a learned boundary-classifier feature (Section~\ref{sec:gbc-aug}) and is used when the response exhibits binary jump structure (Phantom, Star). Both share the same IQN architecture and training loss; Definition~\ref{def:gbc} provides the unified formulation.

The comparator methods are:
\begin{itemize}
  \item \textbf{hetGP}: \texttt{mleHetGP} from the R package \texttt{hetGP}
        \citep{binois2018practical}, Mat\'{e}rn-$5/2$ kernel with jointly estimated latent mean and spatially-varying noise variance.
  \item \textbf{MJGP}: R package \texttt{jumpgp} \citep{flowers2026modular},
        EM clustering + classifier + local GP on augmented inputs.
  \item \textbf{Stationary GP}: Mat\'{e}rn-$5/2$ via maximum marginal likelihood.
  \item \textbf{DGP+ALC}: R package \texttt{deepgp} \citep{sauer2023active},
        2-layer deep GP with elliptical slice sampling, active learning via ALC.
\end{itemize}

Throughout we report three metrics: root mean-squared error (RMSE) for point prediction (or root mean-squared percentage error, RMSPE, when the baseline reports relative errors); continuous ranked probability score (CRPS) for distributional accuracy; and empirical 90\% coverage of predictive intervals (nominal $= 0.90$). CRPS simultaneously rewards calibration and sharpness and applies to any predictive distribution \citep{gneiting2007strictly}, making it our primary metric. We estimate CRPS from $M$ posterior quantile samples via $\widehat{\text{CRPS}}(F,y) = \frac{1}{M}\sum_{m=1}^{M}|q_m - y| - \frac{1}{2M^2}\sum_{m=1}^{M}\sum_{m'=1}^{M}|q_m - q_{m'}|$, where $q_m = F^{-1}(\tau_m)$ are quantiles at a regular grid $\tau_m = m/(M+1)$.

GBC models train on a single NVIDIA A100 GPU; GP baselines run on 16-core CPU (Intel Xeon). Wall-clock comparisons therefore reflect both algorithmic complexity ($\mathcal{O}(n)$ SGD vs.\ $\mathcal{O}(n^3)$ Cholesky) and hardware differences (GPU vs.\ CPU). To separate these effects, the Friedman scaling experiment in Section~\ref{sec:friedman} reports asymptotic growth rates alongside absolute times.

\subsection{Motorcycle data (hetGP comparison)}
\label{sec:moto}

The motorcycle helmet dataset (\texttt{MASS::mcycle}, $n=133$) is a canonical benchmark for heteroskedastic non-stationary emulation: head acceleration as a function of time after a simulated crash, with variance that increases sharply during the main deceleration phase (times 15--35\,ms). We compare GBC to \texttt{hetGP} \citep{binois2018practical}, the purpose-built heteroskedastic GP for this setting, rather than a stationary GP.

We use 50 random 80:20 train/test splits (identical for both methods). GBC trains an ensemble of $K{=}5$ IQNs (each $5{,}000$ epochs, different random seeds) on the same splits; predictions pool $B = 200$ quantile samples per model ($1{,}000$ total). No architecture modification is made relative to any other benchmark.

\begin{table}[H]
  \centering
  \caption{Motorcycle benchmark (\texttt{MASS::mcycle}, $n=133$): mean $\pm$ SE over
    50 replicates. \texttt{hetGP} is purpose-built for heteroskedastic data; GBC-IQN,
    with default loss weights and no custom kernel or noise model, achieves
    marginally lower CRPS.}
  \label{tab:moto}
  \begin{tabular}{lrrr}
    \toprule
    Method & RMSE & CRPS & 90\% Coverage \\
    \midrule
    hetGP (Matern-5/2)      & $\mathbf{23.81 \pm 0.49}$ & $12.56 \pm 0.24$ & $0.875$ \\
    GBC (IQN, single)        & $23.92 \pm 0.57$ & $12.56 \pm 0.29$ & $0.803$ \\
    GBC (IQN, $K{=}5$ ens.)  & $23.88 \pm 0.57$ & $\mathbf{12.49 \pm 0.29}$ & $0.830$ \\
    \bottomrule
  \end{tabular}
\end{table}

Figure~\ref{fig:moto} illustrates three approaches to predictive uncertainty. The stationary GP assigns near-constant interval width across the domain. The \texttt{hetGP} adapts via a latent noise process, producing wider intervals during the high-variance deceleration phase and narrower intervals elsewhere. GBC adapts its interval width to local variance without a parametric noise model --- the IQN's conditional quantile function inherently captures heteroskedasticity. Ensembling smooths the posterior mean and widens the intervals, partially closing the coverage gap with \texttt{hetGP} (from $0.803$ to $0.830$, though still below \texttt{hetGP}'s $0.875$).

The central finding is that GBC, with default loss weights and no custom noise model, is comparable to a purpose-built heteroskedastic GP on CRPS ($12.49$ vs.\ $12.56$; the difference is within standard-error overlap and should not be interpreted as a significant advantage). The \texttt{hetGP} achieves better coverage because its explicit noise model $\sigma^2_\delta(x)$ provides a well-calibrated Gaussian predictive distribution when the heteroskedastic Gaussian assumption holds --- a structural advantage at this small sample size ($n_{\rm train} = 106$) that GBC's nonparametric quantile estimates cannot fully match.

\begin{figure}[H]
  \centering
  \includegraphics[width=\textwidth]{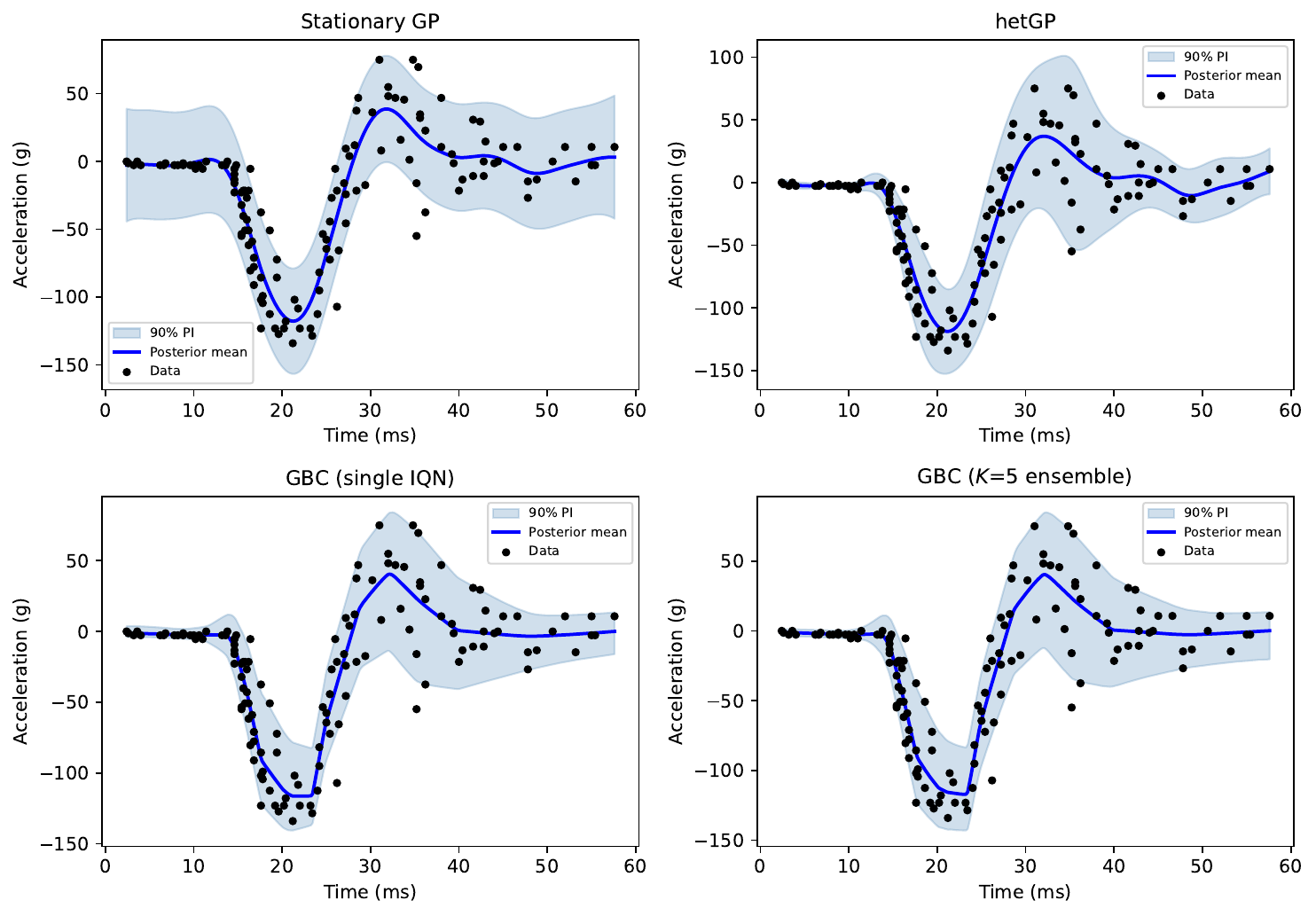}
  \caption{Motorcycle benchmark: stationary GP (top left), \texttt{hetGP}
    (top right), single IQN (bottom left), and $K{=}5$ IQN ensemble
    (bottom right).
    Shaded region: 90\% predictive interval. The stationary GP assigns
    near-constant uncertainty; \texttt{hetGP} adapts via a latent
    noise process; GBC adapts without a parametric noise model.}
  \label{fig:moto}
\end{figure}

\subsection{Piecewise and jump processes (BGP benchmarks)}
\label{sec:bgp}

Piecewise-continuous response surfaces arise whenever a physical system undergoes a regime change --- a phase transition, a structural failure threshold, or the capacity saturation seen in the AMHV data (Section~\ref{sec:amhv}). \citet{park2026active} define Bi-mixture GP (BGP) datasets that isolate this difficulty in a controlled setting:
\[
  f(x) = f_1(x)\,\mathbf{1}_{X_1}(x) + f_2(x)\,\mathbf{1}_{X\setminus X_1}(x),
  \quad x \in [-0.5,0.5]^d,
\]
where $X_1 = \{a^\top x \geq 0\}$ with $a$ drawn uniformly from $\{-1,+1\}^d$; $f_1 \sim \mathcal{GP}(0, 9\,k_{\rm SE})$ and $f_2 \sim \mathcal{GP}(13, 9\,k_{\rm SE})$ with squared-exponential kernel of length scale $0.1d$; and observations are noisy: $y_i = f(x_i) + \varepsilon_i$, $\varepsilon_i \sim \mathcal{N}(0,4)$. The two regimes are smooth GP surfaces separated by a hyperplane whose orientation varies randomly across replicates. The jump-to-noise ratio is $13/2 = 6.5$: large enough that the discontinuity is clearly visible but the noise prevents trivial boundary detection. A stationary GP must choose a single length scale that trades off smoothness within each regime against fidelity at the boundary --- exactly the tension that motivates non-stationary methods. We test $d \in \{2,3,4\}$.

For each dimension $d$, we generate 20 independent BGP realizations (new partition vector $a$ and new GP draws each time). Each realization provides $N = 2{,}000$ points in $[-0.5,0.5]^d$ with an 80:20 train/test split ($n_{\rm train} = 1{,}600$). Both methods use the same training data. The reported variance across replicates reflects both estimation error and variation in problem difficulty (different $a$ and GP draws).

\begin{table}[H]
  \centering
  \caption{BGP benchmarks (Park et al.\ 2026): mean $\pm$ SE over 20 replicates.
    GBC-IQN achieves lower CRPS at all dimensions; the advantage grows with $d$ as
    the GP's global length-scale estimate becomes less effective.}
  \label{tab:bgp}
  \begin{tabular}{lrrrrrr}
    \toprule
    & \multicolumn{2}{c}{GP (Mat\'{e}rn-5/2)} & \multicolumn{3}{c}{GBC (IQN)} \\
    \cmidrule(lr){2-3}\cmidrule(lr){4-6}
    $d$ & RMSE & CRPS & RMSE & CRPS & 90\% Cov.\ \\
    \midrule
    2 & $2.514 \pm 0.053$ & $1.379 \pm 0.024$ & $\mathbf{2.168 \pm 0.028}$ & $\mathbf{1.224 \pm 0.015}$ & $0.894$ \\
    3 & $2.762 \pm 0.045$ & $1.515 \pm 0.022$ & $\mathbf{2.197 \pm 0.018}$ & $\mathbf{1.236 \pm 0.011}$ & $0.885$ \\
    4 & $3.122 \pm 0.060$ & $1.716 \pm 0.030$ & $\mathbf{2.245 \pm 0.017}$ & $\mathbf{1.267 \pm 0.012}$ & $0.845$ \\
    \bottomrule
  \end{tabular}
\end{table}

GBC achieves lower CRPS on every replicate and dimension, with the average improvement growing with $d$: from 14\% RMSE / 11\% CRPS at $d=2$ to 28\% / 26\% at $d=4$ (standard errors in Table~\ref{tab:bgp}). The stationary GP faces a fundamental trade-off: its single length-scale parameter must simultaneously smooth within regions and accommodate the sharp jump; the IQN adapts its conditional quantile function to each side of the partition without a kernel assumption.

Figure~\ref{fig:bgp} illustrates the $d=2$ case. The stationary GP assigns a single length scale to the entire domain, producing blurring of the predicted mean near the boundary. The IQN captures the level shift cleanly, with narrower uncertainty intervals away from the boundary and broader intervals near it.

\begin{figure}[H]
  \centering
  \includegraphics[width=\textwidth]{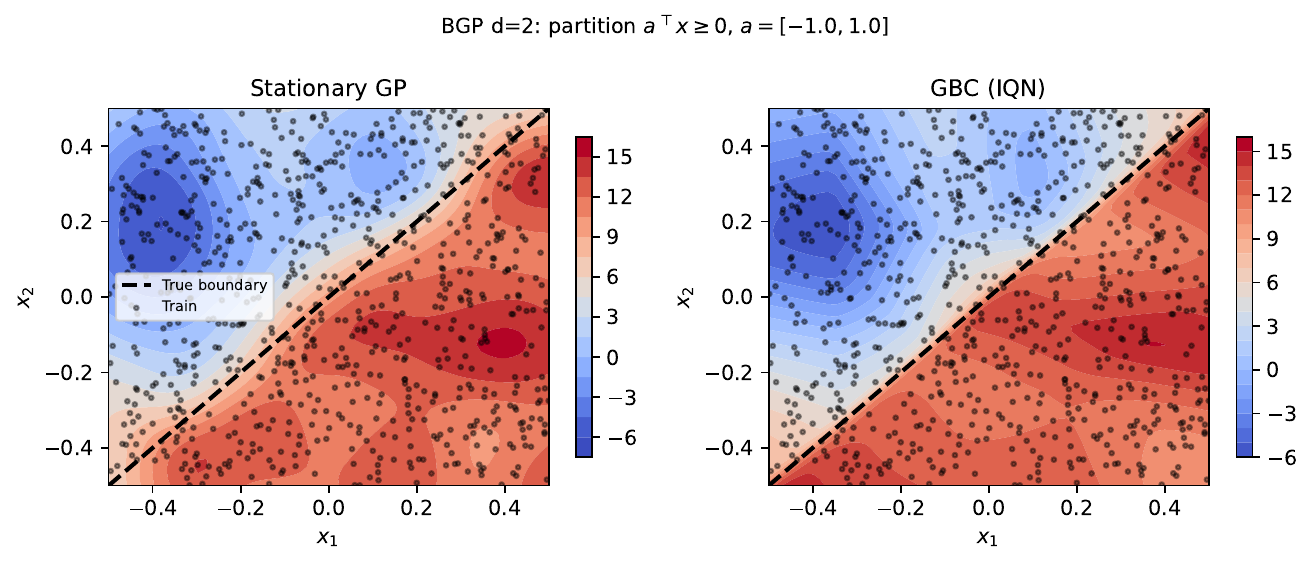}
  \caption{BGP benchmark ($d=2$): predicted mean surface from stationary GP (left)
    vs.\ GBC/IQN (right). Dashed line: true partition boundary $a^\top x = 0$.
    GP smooths across the jump; GBC captures the regime change.}
  \label{fig:bgp}
\end{figure}

\subsection{Friedman high-dimensional scaling}
\label{sec:friedman}

The Friedman function \citep{friedman1991multivariate,surjanovic2013virtual}
\[
  f(x) = 10 \sin(\pi x_1 x_2) + 20(x_3 - 0.5)^2 + 10 x_4 + 5 x_5, \quad x \in [0,1]^{10},
\]
is a standard benchmark for high-dimensional regression: five of the ten inputs are active (with a nonlinear interaction between $x_1$ and $x_2$), and the remaining five are inert, testing a method's ability to concentrate on the relevant subspace. The smooth but nonlinear structure makes this a natural GP problem, and it is widely used to benchmark scalability because GP training cost grows as $\mathcal{O}(n^3)$ while the function itself is cheap to evaluate at any $n$. We evaluate both prediction accuracy (at $n = 2{,}000$ in 10 replicates) and scalability (single runs at $n \in \{500, 1{,}000, 2{,}000, 5{,}000, 10{,}000, 20{,}000\}$).

Training data: $y^{(i)} = f(x^{(i)}) + \varepsilon^{(i)}$, $\varepsilon^{(i)} \sim \mathcal{N}(0,1)$. GP baseline uses Mat\'{e}rn-$5/2$ with noise, optimized via maximum marginal likelihood. GBC trains for $3{,}000$ epochs; both methods are evaluated on $500$ held-out test points.

\begin{table}[H]
  \centering
  \caption{Friedman $d=10$, $n=2{,}000$: mean $\pm$ SE over 10 replicates.
    GBC-IQN achieves lower RMSE and CRPS. Both methods produce coverage well above
    the 90\% nominal level, indicating conservative (overly wide) intervals.}
  \label{tab:friedman}
  \begin{tabular}{lrrr}
    \toprule
    Method & RMSE & CRPS & 90\% Coverage \\
    \midrule
    GP (Matern 5/2) & $0.590 \pm 0.016$ & $0.378 \pm 0.005$ & $0.995$ \\
    GBC (IQN)       & $\mathbf{0.520 \pm 0.009}$ & $\mathbf{0.323 \pm 0.004}$ & $0.991$ \\
    \bottomrule
  \end{tabular}
\end{table}

GBC achieves RMSE~$= 0.520 \pm 0.009$ and CRPS~$= 0.323 \pm 0.004$, versus GP's RMSE~$= 0.590 \pm 0.016$ and CRPS~$= 0.378 \pm 0.005$ --- a 12\% improvement on RMSE and 14\% on CRPS, with GBC achieving lower CRPS on every individual replicate (10/10).

Table~\ref{tab:scaling} and Figure~\ref{fig:scaling} report the scaling behavior. At $n = 2{,}000$, GBC is both faster (17.7\,s GPU vs.\ 47.6\,s CPU) and more accurate. Beyond $n = 2{,}000$, the dense-covariance GP cannot be evaluated on a standard workstation, while GBC continues to improve: RMSE~$= 0.40$ at $n = 5{,}000$, $0.36$ at $n = 10{,}000$, and $0.29$ at $n = 20{,}000$ (training time: 152.5\,s). On this benchmark and hardware configuration, GBC becomes both faster and more accurate at approximately $n = 1{,}500$; the crossover point will vary with problem structure and hardware.

\begin{table}[H]
  \centering
  \caption{Friedman $d=10$ scaling study: training time (seconds) and test RMSE
    at each sample size. GP is infeasible beyond $n = 2{,}000$.
    Each row is a single run; the $n{=}2{,}000$ RMSE values differ from
    Table~\ref{tab:friedman}'s 10-replicate means because of sampling variability
    in the training set.}
  \label{tab:scaling}
  \begin{tabular}{rrrrrr}
    \toprule
    & \multicolumn{2}{c}{GP (Matern 5/2)} & \multicolumn{2}{c}{GBC (IQN)} \\
    \cmidrule(lr){2-3}\cmidrule(lr){4-5}
    $n$ & Time (s) & RMSE & Time (s) & RMSE \\
    \midrule
    500     &   2.9 & 1.26 &   9.6 & 0.91 \\
    1{,}000 &   9.7 & 0.76 &  10.4 & 0.71 \\
    2{,}000 &  47.6 & 0.64 &  17.7 & 0.53 \\
    5{,}000 & \multicolumn{2}{c}{infeasible} &  41.2 & 0.40 \\
    10{,}000 & \multicolumn{2}{c}{infeasible} &  70.6 & 0.36 \\
    20{,}000 & \multicolumn{2}{c}{infeasible} & 152.5 & 0.29 \\
    \bottomrule
  \end{tabular}
\end{table}

\begin{figure}[H]
  \centering
  \includegraphics[width=\textwidth]{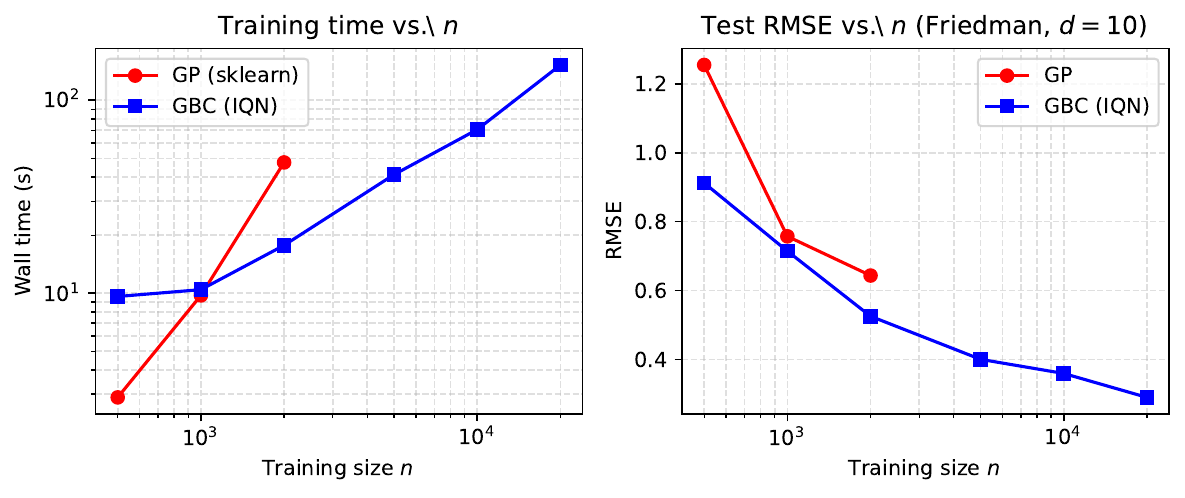}
  \caption{Friedman $d=10$: training time (left) and test RMSE (right) vs.\
    sample size. GP time grows as $\mathcal{O}(n^3)$; GBC grows approximately
    linearly. GBC achieves lower RMSE at all $n$ and remains feasible as $n$
    grows beyond GP's reach.}
  \label{fig:scaling}
\end{figure}

\section{Real-Data and Applied Examples}\label{sec:real}

This section evaluates GBC on five applied benchmarks from the computer experiments literature: semiconductor transport (AMHV), spatial jump processes (Phantom, Star, Michalewicz), and active learning for rocket aerodynamics and satellite drag.

\subsection{AMHV semiconductor transport}
\label{sec:amhv}

The AMHV (Automated Material Handling Vehicle) transport dataset \citep{kang2024bayesian} records average transport time (seconds) for autonomous vehicles in a semiconductor wafer fabrication facility ($N = 9{,}503$, $d=4$). The four inputs --- vehicle speed, acceleration, minimum inter-vehicle distance, and empty-vehicle search range --- control logistics parameters whose combined effect on transport time is smooth at low utilization but jumps sharply when the system approaches capacity saturation. This real-world dataset is the hardest of the four Flowers benchmarks: the two response regimes \emph{overlap} in the marginal distribution (unlike the well-separated modes of Phantom and Star), and $d{=}4$ makes the data sparse relative to input volume. Inputs and response are normalized to $[0,1]$ following \citet{flowers2026modular}. We use the same 90:10 train/test protocol (20 replicates).

On this smooth-boundary dataset, GBC without augmentation achieves RMSE~$= 0.024$ and CRPS~$= 0.011$ (SE ${<}\,0.001$ for both) with 91\% coverage. MJGP achieves CRPS~$= 0.012$, marginally worse, and its RMSE ($0.027$) is actually worse than GBC's ($0.024$) --- the jump-detection step introduces overhead when the boundary is smooth rather than sharp. This illustrates a recurring pattern: a general-purpose architecture avoids the performance penalty that specialized methods may incur when their structural assumptions are not well-matched to the data.

\subsection{Flowers jump datasets (Phantom and Star)}
\label{sec:flowers}

\citet{flowers2026modular} benchmark their MJGP on four datasets representing diverse piecewise structures. The Phantom dataset ($d{=}2$, $N{=}10{,}201$) simulates a bivariate response with a \emph{sinusoidal} jump boundary ($\sin(x_1) = x_2$) that creates nested curvilinear regions --- a shape that defeats axis-aligned partitions and simple classifiers. The Star dataset ($d{=}2$, $N{=}10{,}201$) replaces the smooth sine boundary with an angular, star-shaped partition in which every point in the interior regime is geographically close to the exterior regime, posing a challenge for neighborhood-based methods. Both datasets generate the response from two Gaussian components conditioned on regime membership, yielding a bimodal marginal distribution with a sharp level shift across the boundary. The difficulty is not in predicting the response within each regime (a simple GP suffices there) but in learning the complex boundary geometry: misclassifying even a thin strip of the input space near the boundary produces large errors in both the mean prediction and the predictive interval placement.

We compare GBC directly to MJGP using the same data, the same 90:10 random train/test protocol (20 replicates), and the \emph{full} training set for both methods.

\begin{table}[H]
  \centering
  \caption{Flowers et al.\ (2026) benchmark datasets: mean $\pm$ SE over 20 replicates,
    full training sets. MJGP uses the authors' code \citep{flowers2026modular}.
    GBC-Aug augments the IQN input with a boundary-classifier probability
    (Section~\ref{sec:gbc-aug}).
    \textbf{Bold} = best value across methods.}
  \label{tab:flowers}
  \resizebox{\textwidth}{!}{%
  \begin{tabular}{lrrrrrr}
    \toprule
    & \multicolumn{2}{c}{MJGP} & \multicolumn{2}{c}{GBC} & \multicolumn{2}{c}{GBC-Aug} \\
    \cmidrule(lr){2-3}\cmidrule(lr){4-5}\cmidrule(lr){6-7}
    Dataset ($n_{\rm train}$) & RMSE & CRPS & RMSE & CRPS & RMSE & CRPS \\
    \midrule
    Phantom ($n{=}9\text{k}$, $d{=}2$) & $0.053\pm 0.001$ & $0.010\pm 0.000$ & $0.095\pm 0.001$ & $0.031\pm 0.001$$^{\dagger}$ & $\mathbf{0.044\pm 0.002}$ & $\mathbf{0.009\pm 0.000}$$^{\dagger}$ \\
    Star ($n{=}9\text{k}$, $d{=}2$)    & $0.085\pm 0.001$ & $0.024\pm 0.001$ & $0.124\pm 0.001$ & $0.045\pm 0.001$$^{\dagger}$ & $\mathbf{0.059\pm 0.002}$ & $\mathbf{0.013\pm 0.000}$$^{\dagger}$ \\
    AMHV ($n{=}8.5\text{k}$, $d{=}4$)    & $0.027\pm 0.000$ & $0.012\pm 0.000$ & $\mathbf{0.024\pm 0.000}$ & $\mathbf{0.011\pm 0.000}$ & \multicolumn{2}{c}{---} \\
    Michalewicz ($n{=}90\text{k}$, $d{=}4$) & \multicolumn{2}{c}{infeasible\,$(\geq 48$\,h)} & $\mathbf{0.064\pm 0.001}$ & $\mathbf{0.031\pm 0.000}$$^{\dagger}$ & \multicolumn{2}{c}{---} \\
    \bottomrule
  \end{tabular}%
  }
  \smallskip\\
  {\footnotesize $^{\dagger}$Pinball-dominant loss ($w_1{=}0.10$,
    $w_2{=}0.20$, $w_3{=}0.70$); AMHV uses default loss.
    GBC 90\% coverage: Phantom 0.91, Star 0.89, AMHV 0.91, Michalewicz 0.91.
    GBC-Aug applied only to datasets with binary jump structure (Phantom, Star).
    MJGP metrics are our evaluations using the authors' code;
    \citet{flowers2026modular} report RMSPE only.
    ``$\pm 0.000$'' indicates SE $< 0.0005$.}
\end{table}

On Phantom ($d{=}2$), plain GBC with quantile-dominant loss achieves CRPS~$= 0.031$ --- approximately $3\times$ worse than MJGP's $0.010$. The $L_1$ location-anchor term encourages the shared hidden layers to predict the conditional median, which near a jump boundary is an average of the two regimes rather than a sharp transition. Reducing the anchor weight to $w_1{=}0.10$ and increasing the quantile weight to $w_3{=}0.70$ frees the network to learn sharper boundary representations.  A loss-weight ablation on Phantom (20 replicates, 6{,}000 epochs) confirms this: $(0.05, 0.25, 0.70)$ gives CRPS~$= 0.033$; $(0.00, 0.30, 0.70)$ gives $0.035$; and $(0.10, 0.20, 0.70)$ gives $0.031$ (Table~\ref{tab:flowers}). Removing the anchor entirely ($w_1 = 0$) degrades performance, suggesting that a small $L_1$ term stabilizes optimization without biasing the quantile estimates toward the boundary average. The largest gain --- $28\%$ on Phantom --- comes from this loss reweighting alone, with no architectural change; the gain decreases with boundary sharpness ($15\%$ on Star, $9\%$ on Michalewicz, $5\%$ on the smooth AMHV).

Applying GBC-Aug (Section~\ref{sec:gbc-aug}) closes the remaining gap entirely. With an MLP classifier achieving 99.98\% accuracy on Phantom, GBC-Aug achieves RMSE~$= 0.044 \pm 0.002$ (vs.\ MJGP's $0.053$) and CRPS~$= 0.009$ (SE ${<}\,0.001$) with 91\% coverage, matching MJGP on CRPS and improving RMSE by 17\%. On Star, GBC-Aug achieves CRPS~$= 0.013$ (SE ${<}\,0.001$; 89\% coverage), outperforming MJGP's $0.024$ by 46\%. Total training time (classifier + IQN) is $\approx 23$\,s per replicate on GPU (A100), compared to $> 600$\,s for MJGP on 16 CPU threads. This $26\times$ wall-clock ratio reflects both the $\mathcal{O}(n)$ vs.\ $\mathcal{O}(n^3)$ algorithmic difference and the GPU-vs-CPU hardware asymmetry.

A notable finding is that ensembling \emph{hurts} on jump data. On smooth data (motorcycle), pooling $K{=}5$ IQNs improves CRPS via variance reduction (Table~\ref{tab:moto}). On jump data, ensembling increases CRPS --- the opposite effect. The practical recommendation: for jump data, use a single model with quantile-dominant loss; for smooth data, ensemble for variance reduction. We discuss a candidate mechanism (boundary-location blur across members) in Section~\ref{sec:discussion}.

The Michalewicz benchmark \citep{surjanovic2013virtual} ($d{=}4$, $N{=}100{,}000$ via Latin hypercube; 90:10 split, $n_{\rm train}{=}90{,}000$) modifies the standard Michalewicz function by adding a $0.5$ offset to the flat region, introducing a jump manifold in four-dimensional input space. Unlike Phantom and Star, the marginal response is not bimodal: it is a point mass near $0.5$ plus a scattered distribution in $[-4, 0]$, violating the Gaussian-mixture clustering assumption. At $n_{\rm train} = 90{,}000$, MJGP is computationally infeasible within practical time limits ($\geq 48$\,h). GBC trains in under 30\,min on GPU, achieving RMSE~$= 0.064 \pm 0.001$ and CRPS~$= 0.031$ (SE ${<}\,0.001$) --- a regime accessible only to $\mathcal{O}(n)$ methods.

\subsection{Rocket LGBB active learning}
\label{sec:al-rocket}

The Rocket LGBB (Langley Glide-Back Booster) benchmark \citep{pamadi2004aerodynamic,gramacy2020surrogates,sauer2023active} models the aerodynamics of a reusable NASA rocket booster gliding back to Earth after launch. The three inputs are Mach number, angle of attack, and side-slip angle; the response is the side-force coefficient (one of six aerodynamic coefficients produced by the simulator). The sound barrier at Mach~$= 1$ imparts a sharp regime change on the response surface, creating a nonlinear ridge in the Mach--angle-of-attack plane that stationary GP kernels systematically oversmooth. This makes the LGBB a standard benchmark for non-stationary surrogates in the computer experiments literature \citep{gramacy2020surrogates}. Following \citet{sauer2023active} exactly, we use $n_0{=}50$ initial Latin hypercube points, acquire 250 additional points via the ensemble disagreement criterion, and evaluate on a 1{,}000-point test set after each acquisition (30 independent replicates).

\begin{table}[H]
  \centering
  \caption{Active learning benchmarks: GBC-AL vs.\ DGP+ALC.
    Metric $\pm$ SE over independent replicates.
    Ratio $>1$ favors GBC (bold).}
  \label{tab:al}
  \begin{tabular}{lrcccr}
    \toprule
    Dataset & $n$ & Metric & GBC-AL & DGP+ALC & Ratio \\
    \midrule
    2D Exponential  & 100 & RMSE  & $0.087\pm 0.003$ & $0.072^{\ast}$ & $0.83\times$ \\
    Rocket LGBB     & 300 & RMSE  & $\mathbf{0.0072\pm 0.0002}$ & $0.0213\pm 0.001^{\ast}$ & $\mathbf{2.95\times}$ \\
    Satellite GRACE & 500 & RMSPE & $1.19\pm 0.01$ & $0.967^{\ast}$ & $0.81\times$ \\
    SatMini ($d{=}7$) & 60 & RMSPE & $\mathbf{5.30\pm 0.15}$ & $7.91\pm 0.28$ & $\mathbf{1.49\times}$ \\
    7D Proxy        & 70  & RMSE  & $\mathbf{0.344\pm 0.012}$ & $0.558\pm 0.007$ & $\mathbf{1.62\times}$ \\
    \bottomrule
  \end{tabular}
  \smallskip\\
  {\footnotesize $^{\ast}$DGP values from the published results of \citet{sauer2023active}
    (mean only; SE not available from published output);
    SatMini and 7D Proxy DGP values from our DGP+ALC runs on the same data splits.
    Replicates: 30 (2D, Rocket GBC), 50 (Rocket DGP, from published code), 10 (Satellite, SatMini, 7D Proxy).
    Supplement~D provides details on 2D Exponential, SatMini, and 7D Proxy.
    Coverage is not reported for AL benchmarks because the evaluation
    metric is point-prediction RMSE/RMSPE, following \citet{sauer2023active}.}
\end{table}

Figure~\ref{fig:al-rocket} shows the learning curves. At the final budget, GBC-AL achieves $2.95\times$ lower RMSE than DGP+ALC ($0.0072$ vs.\ $0.0213$), with non-overlapping replicate distributions (every GBC run yields lower RMSE than every DGP run); the 90\% bands separate by $n \approx 150$. The Rocket response surface contains sharp nonlinearities in the side-force coefficient that are better captured by the IQN than by the DGP's smooth squared-exponential kernel.

\begin{figure}[H]
  \centering
  \includegraphics[width=0.8\textwidth]{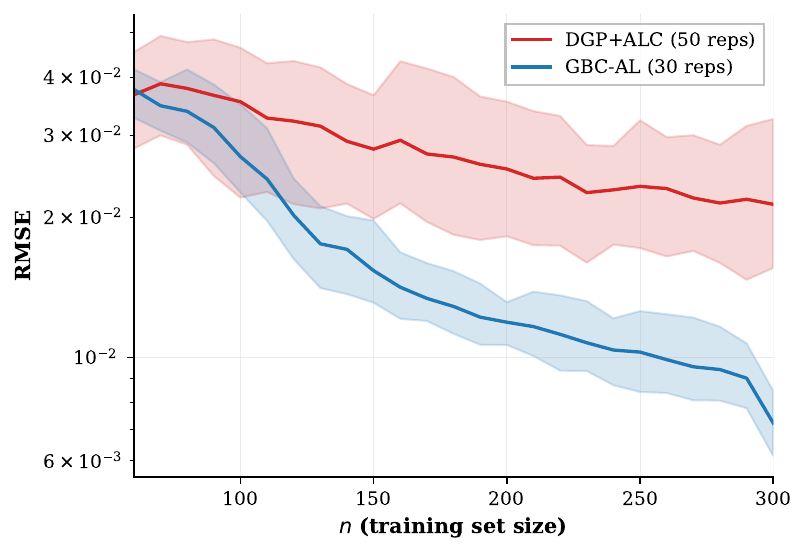}
  \caption{Rocket LGBB active learning: RMSE vs.\ training set size. Lines show
    the mean over 30 replicates (GBC-AL) and 50 replicates (DGP+ALC);
    shaded regions are 90\% bands (5th to 95th percentile).
    DGP+ALC replicate count (50) reflects the published experimental
    design of \citet{sauer2023active}; GBC-AL uses 30 replicates.
    The unequal counts do not alter the qualitative conclusion: the bands
    separate completely by $n \approx 150$, and every GBC replicate achieves lower RMSE
    than every DGP replicate at $n = 300$.}
  \label{fig:al-rocket}
\end{figure}

On the 2D Exponential benchmark ($d{=}2$, smooth, $n{=}100$), DGP achieves 21\% lower RMSE. This is consistent with the GP's inductive bias being well-matched to smooth, low-dimensional functions: the squared-exponential kernel's smoothness assumption provides effective regularization from very few points, while the neural network must estimate its implicit smoothness from limited data. Two supplementary proxy experiments (Supplement~D) corroborate these findings at smaller budgets: SatMini, a reduced-budget version of the Satellite GRACE problem ($n{=}60$), and 7D Proxy, a synthetic smooth function in $[0,1]^7$ designed to mimic the satellite drag surface at low cost. GBC-AL outperforms DGP+ALC on both ($+33\%$ and $+38\%$, Table~\ref{tab:al}).

\subsection{Satellite GRACE active learning}
\label{sec:al-sat}

The Satellite GRACE benchmark \citep{sun2019emulating,mehta2014modeling, sauer2023active} uses a Test Particle Monte Carlo (TPMC) simulator developed at Los Alamos National Laboratory to compute drag coefficients for satellites in low-Earth orbit --- a quantity needed for positioning and collision avoidance. The seven inputs describe the thermospheric conditions (atmospheric composition, velocity, temperature) encountered by the GRACE satellite. Unlike the Rocket LGBB, the response surface is smooth and the simulator is \emph{stochastic} ($\approx 86$\,s per evaluation), making this a high-dimensional, expensive-to-evaluate problem where the GP's smoothness prior is well-matched. \citet{mehta2014modeling} show that a GP trained on a static 1{,}000-point Latin hypercube achieves $< 1\%$ RMSPE; the active learning question is whether sequential design can reach this accuracy with fewer runs. Following \citet{sauer2023active}, we fix $n_0{=}110$ initial points (including 10 replicate pairs to estimate noise) and acquire 390 additional points (10 independent replicates).

Figure~\ref{fig:al-sat} reveals a nuanced comparison. GBC-AL leads throughout the early acquisition phase: at $n{=}200$, GBC achieves RMSPE~$= 1.89 \pm 0.02$ (mean $\pm$ SE over 10 replicates) versus DGP's $2.26$ (from published learning curve), a 16\% advantage. However, after approximately $n{=}300$ points, DGP overtakes GBC, ultimately reaching RMSPE~$= 0.967$ at $n{=}500$ compared to GBC's $1.19$. This crossover is consistent with the pattern observed across all benchmarks: GP surrogates benefit from an informative smoothness prior that the neural network lacks, and this advantage grows with $n$ on smooth surfaces (Section~\ref{sec:discussion}).

\begin{figure}[H]
  \centering
  \includegraphics[width=0.8\textwidth]{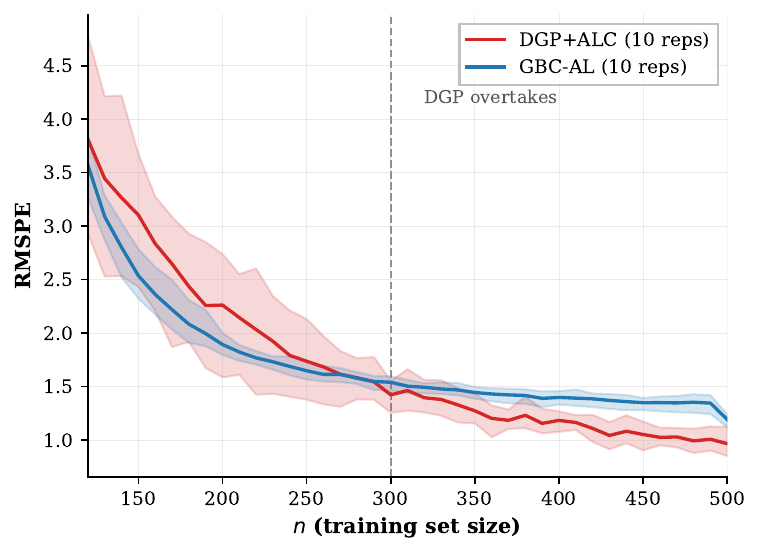}
  \caption{Satellite GRACE active learning: RMSPE vs.\ training set size.
    GBC-AL leads for $n < 300$; DGP+ALC overtakes after the crossover
    (dashed vertical line).
    For expensive simulators where the acquisition budget is limited,
    GBC provides better accuracy during the early phase.}
  \label{fig:al-sat}
\end{figure}

For expensive simulators where the acquisition budget is limited --- the most practically relevant regime --- GBC-AL delivers better accuracy during the early phase when each new point matters most. The SatMini and 7D Proxy benchmarks in Table~\ref{tab:al} confirm this small-$n$ advantage on independent datasets: GBC improves by 33\% ($n{=}60$) and 38\% ($n{=}70$), respectively.

These results suggest that GBC-AL tends to be more effective when the acquisition budget is moderate relative to the input dimension, while DGP+ALC benefits from its informative smoothness prior when enough data accumulates to identify the length-scale parameters of a smooth kernel.

Table~\ref{tab:summary} consolidates GBC's performance across the benchmarks studied. In the emulation comparisons (Table~\ref{tab:summary}, top), GBC matches or improves CRPS in every setting where a direct GP comparison is available, with the advantage growing with problem dimension and sample size.

\begin{table}[H]
  \centering
  \caption{Summary of GBC performance across all benchmarks. ``GP baseline''
    denotes the method-of-choice GP variant for each problem type.
    $\Delta$ metric is the percentage improvement of GBC over the GP baseline
    (positive = GBC better; CRPS for emulation, RMSE/RMSPE for AL);
    $k\times$ denotes a ratio when the improvement exceeds 100\%.
    \textbf{Bold} indicates GBC is better.}
  \label{tab:summary}
  \small
  \resizebox{\textwidth}{!}{%
  \begin{tabular}{llrrrl}
    \toprule
    Benchmark & GP Baseline & GP Metric & GBC Metric & $\Delta$ & Note \\
    \midrule
    \multicolumn{6}{l}{\textit{Emulation (CRPS, lower is better)}} \\
    BGP $d{=}2$ ($n{=}1.6$k) & Stationary GP & 1.379 & \textbf{1.224} & $\mathbf{+11\%}$ & 20/20 reps \\
    BGP $d{=}3$ ($n{=}1.6$k) & Stationary GP & 1.515 & \textbf{1.236} & $\mathbf{+18\%}$ & 20/20 reps \\
    BGP $d{=}4$ ($n{=}1.6$k) & Stationary GP & 1.716 & \textbf{1.267} & $\mathbf{+26\%}$ & 20/20 reps \\
    Motorcycle ($n{=}133$) & hetGP & 12.56 & \textbf{12.49} & $\mathbf{+0.5\%}$$^{\S}$ & 50 reps \\
    Friedman $d{=}10$ ($n{=}2$k) & Stationary GP & 0.378 & \textbf{0.323} & $\mathbf{+14\%}$ & 10/10 reps \\
    Friedman ($n{=}20$k) & GP infeasible & --- & \textbf{0.29} & --- & scales to $n{=}90$k \\
    Michalewicz ($n{=}90$k) & MJGP infeasible & --- & \textbf{0.031} & --- & GBC only \\
    Phantom $d{=}2$ ($n{=}9$k) & MJGP & 0.010 & \textbf{0.009}$^{\ddagger}$ & $\approx$\,$\mathbf{+10\%}$ & GBC-Aug; marginal \\
    Star $d{=}2$ ($n{=}9$k) & MJGP & 0.024 & \textbf{0.013}$^{\ddagger}$ & $\mathbf{+46\%}$ & GBC-Aug \\
    AMHV $d{=}4$ ($n{=}8.5$k) & MJGP & 0.012 & \textbf{0.011} & $\mathbf{+8\%}$ & smooth \\
    \midrule
    \multicolumn{6}{l}{\textit{Active learning (RMSE/RMSPE, lower is better)}} \\
    AL 2D Exp ($n{=}100$) & DGP+ALC & \textbf{0.072} & 0.087 & $-21\%$ & smooth, small $n$ \\
    AL Rocket ($n{=}300$) & DGP+ALC & 0.0213 & \textbf{0.0072} & $\mathbf{2.95\times}$ & 30/30 reps \\
    AL Satellite ($n{<}300$) & DGP+ALC & 2.26$^{\dagger}$ & \textbf{1.89}$^{\dagger}$ & $\mathbf{+16\%}$ & GBC leads \\
    AL Satellite ($n{=}500$) & DGP+ALC & \textbf{0.967} & 1.19 & $-23\%$ & DGP leads \\
    AL SatMini ($n{=}60$) & DGP+ALC & 7.91 & \textbf{5.30} & $\mathbf{+33\%}$ & small $n$ \\
    AL 7D Proxy ($n{=}70$) & DGP+ALC & 0.558 & \textbf{0.344} & $\mathbf{+38\%}$ & small $n$ \\
    \bottomrule
  \end{tabular}%
  }
  \smallskip\\
  {\footnotesize\textit{Note:} $^{\ddagger}$GBC-Aug (boundary-augmented); see Section~\ref{sec:gbc-aug}.
    $^{\dagger}$At $n{=}200$; GBC leads for $n < 300$, DGP overtakes after.
    GBC records a favorable point estimate in 12 of 14 direct comparisons;
    2 additional rows are feasibility-only (GP infeasible).
    Comparisons are descriptive: protocols, replicate counts (10--50), and
    baseline sources (re-run vs.\ published) vary across benchmarks
    (see Table~\ref{tab:al} footnote).
    DGP is better on smooth surfaces at small $n$ (2D Exp) and large $n$ (Satellite).
    $^{\S}$Within standard-error overlap; not a significant difference.}
\end{table}

\section{Discussion}\label{sec:discussion}

The CRPS advantage grows with problem dimension (11\% to 26\% on the BGP benchmarks) and with deviation from stationarity. Against specialist methods, GBC-Aug matches or surpasses MJGP by adopting its boundary-detection pipeline with an IQN backend; GBC-AL outperforms DGP+ALC on Rocket ($2.95\times$) but loses on smooth surfaces at large~$n$ (Satellite GRACE). The two cases where GBC consistently underperforms --- 2D Exponential and Satellite at $n{=}500$ --- both involve smooth functions where the GP's squared-exponential kernel provides effective regularization that the neural network does not replicate from limited data.

Ensembling helps on smooth data but hurts on jump data (Section~\ref{sec:flowers}). A candidate mechanism is that ensemble members localize the discontinuity boundary at slightly different locations, and pooling their quantile samples blurs the conditionals near the boundary. We have not isolated this effect experimentally (e.g., via boundary-location ablation), so it remains a hypothesis.

GBC also has a natural connection to likelihood-free inference. Because the IQN is trained on simulated pairs $(\theta, y)$ drawn from the joint distribution, it functions as an amortized posterior sampler in the sense of simulation-based inference \citep{cranmer2020frontier,papamakarios2021normalizing}. Unlike ABC \citep{beaumont2002approximate}, which requires per-observation forward-model evaluations, GBC generates posterior samples via a single forward pass after training. Unlike normalizing flows, the IQN provides direct access to arbitrary quantiles without invertibility constraints. Supplement~C demonstrates this connection on a bimodal posterior inference task.

GBC training is $\mathcal{O}(N \cdot L \cdot w^2)$ where $L$ is the number of layers and $w$ the layer width; this is linear in $N$ via mini-batch SGD. At inference, cost is $\mathcal{O}(B \cdot L \cdot w^2)$ per new observation, independent of training-set size.  Because GBC and GP baselines run on different hardware (GPU vs.\ CPU), wall-clock comparisons reflect both algorithmic and hardware differences; Figure~\ref{fig:scaling} isolates the asymptotic growth rates.

IQN calibration varies across problem regimes. Coverage is near-nominal ($0.85$--$0.91$) on piecewise and heteroskedastic benchmarks (BGP, AMHV), under-nominal on motorcycle ($0.83$, small $n$), and over-conservative on Friedman ($> 0.99$). No existing theory predicts which regime applies; the answer depends on the interaction between signal-to-noise ratio, boundary sharpness, and sample size. The CRPS improvements we report mix calibration and sharpness; a PIT decomposition of these components remains an open question.

More broadly, neural network variance estimates --- whether from heteroskedastic NLL, ensemble spread, or last-layer Laplace approximations --- are systematically overconfident on out-of-sample data: the predicted variance is calibrated on training inputs but does not extrapolate to regions where the network has seen little data.  When guaranteed coverage is required, conformal prediction \citep{vovk2005algorithmic,romano2019conformalized} provides a distribution-free post-processing step that recalibrates any black-box predictor using held-out residuals.  The IQN's quantile-based intervals avoid the variance-estimation problem but, as shown above, still exhibit regime-dependent miscalibration.

Several caveats apply. The benchmarks do not cover all domains ($d > 10$, functional outputs, stochastic simulators). Some gains are within standard-error overlap (motorcycle $+0.5\%$, Phantom $0.009$ vs.\ $0.010$). The ``12 of 14'' count is descriptive, not a formal test; replicate counts vary (10--50) and some baselines lack standard errors. Wall-clock ratios reflect GPU-vs-CPU hardware differences in addition to algorithmic scaling; Figure~\ref{fig:scaling} isolates the latter. Two baseline classes are absent: sparse variational GPs \citep{titsias2009variational,hensman2013gaussian}, which address cubic cost but retain stationarity/Gaussian assumptions, and non-GP surrogates (e.g., gradient-boosted trees), which lack principled distributional output for CRPS comparison.

Several open problems remain: formal AL optimality theory for IQNs analogous to IMSE/ALC for GPs \citep{sacks1989design,gramacy2020surrogates}; adaptive architecture selection; extension to vector-valued outputs; and improved epistemic uncertainty for small-$n$ active learning, where GP posteriors provide a structural advantage.

For practitioners, the key guideline is that GBC's strengths and limitations are predictable from problem structure. GBC is the stronger choice when the response has discontinuities, when $d$ or $n$ is large enough that cubic GP cost becomes prohibitive, or when the predictive distribution is non-Gaussian. A well-tuned GP remains preferable on smooth, moderate-$n$ surfaces ($n < 5{,}000$, $d \leq 4$) where the kernel prior provides effective regularization. The two methods can also be combined: GBC-Aug already borrows the GP literature's EM clustering idea, and conformal calibration \citep{romano2019conformalized} can post-process GBC intervals when formal coverage guarantees are needed.

GBC via IQNs provides a practical surrogate framework that delivers full predictive distributions at $\mathcal{O}(n)$ training cost.  On the benchmarks studied, it matches or improves upon GP surrogates for non-stationary, high-dimensional, and large-$n$ problems, while GPs retain an advantage on smooth, moderate-$n$ surfaces where their kernel prior provides effective regularization.  GBC is not a universal replacement for GPs but a complementary tool whose strengths and weaknesses are predictable from problem structure: discontinuities, high dimensionality, and large training sets favor GBC; smooth functions at small $n$ favor GPs.


\ifanonymous\else

\fi



\bibliography{gbc-gp}
\bibliographystyle{plainnat}

\end{document}